\documentclass[runningheads]{llncs}
\usepackage[T1]{fontenc}
\usepackage{graphicx,verbatim}

\begin{document}
\title{Cascade Detector Analysis and Application to Biomedical Microscopy}
%
\author{Thomas L. Athey \and
Shashata Sawmya \and
Nir Shavit}
\authorrunning{T. Athey et al.}
%
\institute{MIT CSAIL, Cambridge, MA, USA\\
\email{\{tathey\_1,shashata,shanir\}@mit.edu}}



\maketitle              
\begin{abstract}
As both computer vision models and biomedical datasets grow in size, there is an increasing need for efficient inference algorithms. We utilize cascade detectors to efficiently identify sparse objects in multiresolution images. Given an object's prevalence and a set of detectors at different resolutions with known accuracies, we derive the accuracy, and expected number of classifier calls by a cascade detector. These results generalize across number of dimensions and number of cascade levels. Finally, we compare one- and two-level detectors in fluorescent cell detection, organelle segmentation, and tissue segmentation across various microscopy modalities. We show that the multi-level detector achieves comparable performance in 30-75\% less time. Our work is compatible with a variety of computer vision models and data domains.


\end{abstract}

\section{Introduction}

Image data has driven significant discoveries in many disciplines including biology, earth sciences, and medicine. The dimensionality of images is incredible – the number of pixels in an image of a complete mouse brain can be of the same magnitude as the number of stars in the Milky Way. As image acquisition accelerates in research, with some datasets at the petabyte scale \cite{shapson-coe_petavoxel_2024}, there is an increasing need to design algorithms that can process images efficiently.

This work is motivated by the observation of two themes in biomedical image analysis. The first is object sparsity, such as sparse neuron labeling in brain mapping \cite{winnubst_reconstruction_2019,lin_cell-type-specific_2018}, cell sparsity in liquid based cytology \cite{hussain_liquid_2020} and hemorrhage sparsity in CT scans \cite{piao_intracerebral_2023}. The second theme is that large images are often stored at multiple resolutions \cite{hider_brain_2022,kenney_brain_2024,burns_open_2013,wang_managing_2012}. Thus, cascade detectors, which use low-resolution information to quickly rule out background regions, are a natural choice for biomedical image analysis (Fig. \ref{fig:diagram}). In this work, we show how a cascade detector model can offer comparable accuracy in less time than a detector that operates only on the highest resolution. While previous work has focused on two-dimensional images and a fixed number of cascade levels, we extend these results across domain dimensionality and the cascade level number. We demonstrate this speedup in cell body (soma) detection in fluorescence imaging, organelle segmentation in electron microscopy (EM), and tissue segmentation in digital pathology. This is, to our knowledge, the most broad application of cascade detectors to biomedical microscopy. This work could be combined with other efficient inference methods such as network sparsification and quantization to further accelerate image analysis.

\begin{figure}
\centering
\includegraphics[width=\textwidth]{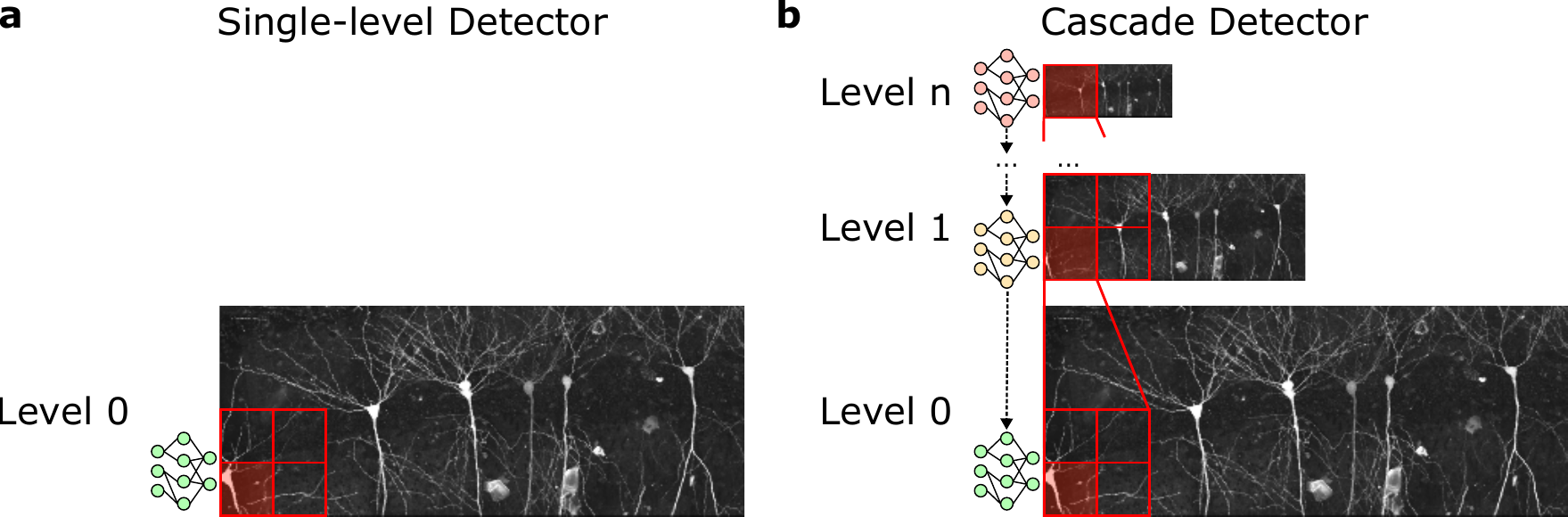}
\caption{Cascade detectors use different levels of a multiresolution image pyramid during inference. \textbf{a)} Schematic of a single-level detector in microscopy where a large image is broken into chunks due to memory limitations. \textbf{b)} In a cascade detector, low resolution data is processed first to rapidly rule out background regions. Lower level detectors are only called on positive candidates. Pictured is a subset of the fluorescence microscopy dataset from Bloss et al. \cite{bloss_structured_2016}.} \label{fig:diagram}
\end{figure}

\section{Related Works}

The subject of this work is performing some computer vision task such as image segmentation or object detection on a large, multiresolution image. Usually the imaging system acquired the highest resolution data, and lower resolutions were computed via some downsampling operation \cite{hider_brain_2022}. In order to avoid possible information loss, computer vision algorithms such as neural networks are often applied to the highest resolution. Additionally, the high resolution data is often broken into chunks due to memory limitations. The goal of our work is how to navigate through the chunks of data to efficiently accomplish a computer vision task. In this way, our work can be connected to active sensing \cite{varotto_active_2021}.

Our approach is inspired more directly by the work of Viola and Jones \cite{viola_robust_2004} which introduces the concept of a classifier cascade and applies them to face detection. Subsequent theoretical work described how to train these systems optimally and how to estimate generalization performance \cite{brubaker_towards_2006}. Neural networks were then used in a cascade framework, again to detect faces \cite{li_convolutional_2015}. The term ``cascade'' has also been used in other contexts such as the application of attention to multi-scale features \cite{rahman_medical_2023}, or the splitting of feature maps across attention heads \cite{liu_multi-branch_2024}.

\section{Theoretical Motivation}

We analyze a setup that is common in contemporary neuroscience research, a three dimensional image provided at multiple resolutions, with a relative downsampling factor of two. Specifically, we consider two resolutions where the high resolution is called level 0 (L0) and the low resolution is called level 1 (L1). For simplicity, we assume that the images at both levels are decomposed into a whole number of uniformly sized chunks. Say the number of L0 chunks is $n$ Further, we assume that the L0 image has exactly twice as many pixels along every axis and thus there are $\frac{n}{8}$ chunks in the L1 image.  The ultimate goal is to determine which L0 chunks contain some object of interest. 

We consider two detector systems. The first is a traditional, single-level detector which classifies all $n$ L0 chunks  (Fig. \ref{fig:diagram}a). The second is a cascade detector which is composed of both a L1 and L0 classifier. The cascade detector classifies all $\frac{n}{8}$  L1 chunks. If a L1 chunk is predicted to be positive, the associated L0 chunks are passed to the L0 detector (Fig. \ref{fig:diagram}b). Thus, a positive prediction at both levels is necessary for a L0 chunk to be predicted as positive.

Say that the true labels of the L0 chunks are independent Bernoulli random variables with probability $p$. Further, say the predictions by the different classifiers are independent of each other when conditioned the true chunk labels. The true positive rate (TPR) and false positive rate (FPR) of the classifier at level $k$ will be denoted $\beta_k$ and $\alpha_k$ respectively. The single-level detector has TPR and FPR $\beta_0$ and $\alpha_0$ by definition. Further,

\begin{proposition}
\label{prop:sens}
Under the setting described above, the two level cascade detector has true positive rate $\beta_{1,0}$ and false positive rate $\alpha_{1,0}$, where

\begin{equation}
    \beta_{1,0}= \beta_1 \beta_0
\end{equation}

\begin{equation}
    \alpha_{1,0}= (1-p)^7(\alpha_1\alpha_0) + (1-(1-p)^7)(\beta_1\alpha_0) 
\end{equation}

Additionally, if $K_n$ is the number of calls to the L0 classifier, then
    \begin{equation}
        E[K_n]= \left(\beta_1+(1-p)^8(\alpha_1-\beta_1)\right)n
    \end{equation}
    \label{prop:acc}
\end{proposition}
%
%
\begin{proof}
The first two equations follow from Bayes rule. A true positive occurs when both the L1 and L0 chunks are correctly predicted to be positive which happens with probability $\beta_1\beta_0$.

The term $(1-p)^7$ is the probability that all seven other L0 chunks contained within the L1 chunk are negative. In this case, $\alpha_1\alpha_0$ is the probability that both detectors give false positives. Under the complement event, $\beta_1\alpha_0$ represents the probability that the L1 prediction is a true positive and the L0 prediction is a false positive.

For the third equation, we assumed that the volume can be split into a whole number of chunks, so we only need to consider the case of a single L1 chunk, then scale the result accordingly. We use the law of total expectation over the intermediate event of whether at least one of the L0 chunks is positive:

    \begin{equation}
        E[K_8]=(1-p)^8(8  \alpha_1) + (1-(1-p)^8)(8  \beta_1) 
    \end{equation}

    Finally, we divide by 8 to normalize for the number of L0 subvolumes.
\end{proof}

Figure \ref{fig:sensitivity} shows the dependence of these equations on various parameters. 

\begin{figure}
\includegraphics[width=\textwidth]{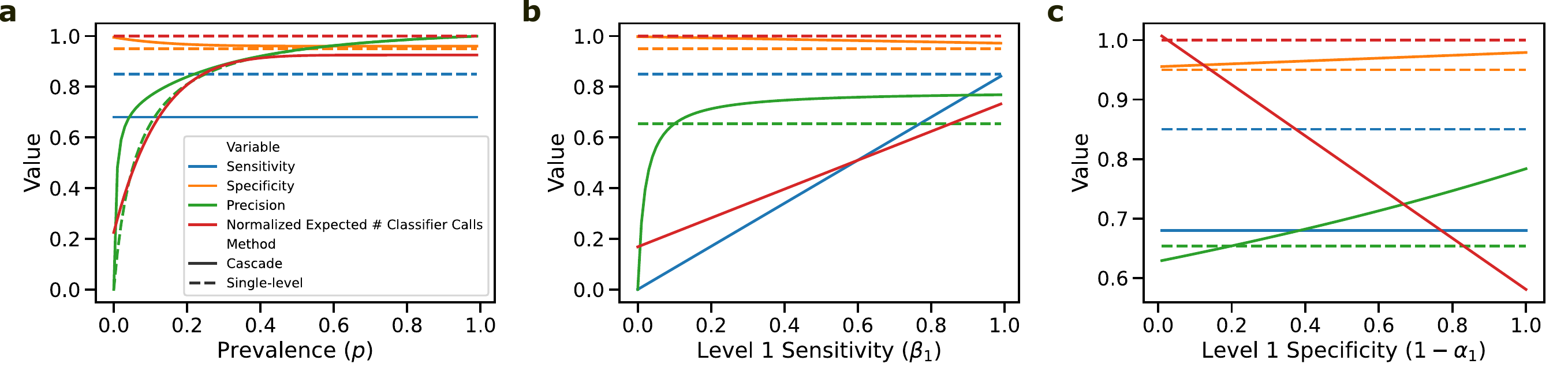}
\caption{The performance of the cascade detector depends on the accuracies of the detectors at both levels, and the prevalence of the object of interest. We performed a sensitivity analysis of these parameters by setting them to a set of values ($\beta_0=0.85$, $\beta_1=0.8$, $\alpha_0=0.05$, $\alpha_1=0.1$, $0=0.1$), then individually varying $p$ (\textbf{b}), $\beta_1$ (\textbf{c}) and $1-\alpha_1$ (\textbf{d}). The detectors' sensitivity, specificity, precision, and expected number of calls to any classifier (normalized by the number of L0 chunks) are plotted.} \label{fig:sensitivity}
\end{figure}

It is trivial to extend these results in the following ways:

\begin{remark}
    Proposition \ref{prop:acc} can be extended to signals over $d$ dimensions by replacing the exponents with $2^d-1$ and $2^d$ respectively.
    \label{remark:dim}
\end{remark}

\begin{remark}
The formulas for sensitivity and specificity can be extended to cascade detectors with more than two levels by using Proposition \ref{prop:acc} recursively. 
\end{remark}

\section{Experiments}
We applied single-level detectors, and a two-level cascade detectors to a variety of publicly available datasets. In each case, the L0 detector in the cascade method is identical to the single-level detector.

\subsection{Datasets}

\subsubsection{Soma detection in fluorescence images}
The ``CA1 somas'' dataset is a three-dimensional fluorescence image of pyramidal neurons in mouse hippocampus \cite{bloss_structured_2016}. Seven ground truth cell bodies were manually annotated in the test set.

The ``fMOST somas'' dataset is a subset of an fMOST image of a whole mouse brain \cite{zeng_whole_2024}. The train and train sets were from different hemispheres. Eighteen ground truth cell bodies were manually annotated in the test set.

\subsubsection{Organelle segmentation in electron microscopy}

In this setting, the task is two-dimensional semantic segmentation
 of organelles in an image stack.

The ``C. elegans mitochondria'' dataset involved sagittal sections of the adult C. elegans head, imaged with EM and segmented for mitochondria \cite{witvliet_connectomes_2021}. Slices 256-319 were used for training, slices 320-335 were used for validation, and slices 0-199 were used for testing.

The ``HeLa nucleus'' dataset is an EM volume, along with nuclei segmentations, from the hela-3 sample in the OpenOrganelle project \cite{heinrich_whole-cell_2021}. Slices 0-49 were used for training, slices 50-59 were used for validation, and slices 60-124 were used for testing.

\subsubsection{Tissue segmentation in digital pathology}

In this setting, the task was tissue segmentation in pathology whole slide images from the CAMELYON dataset \cite{alexi_supporting_2018}. Specifically, we used test images 1-10. Further dataset details are in Table \ref{tab:data}.

\begin{table}
\centering
\caption{Dataset details, including resolution and size of the train/test sets (in pixels).}\label{tab:data}
\begin{tabular}{|p{0.12\linewidth} | p{0.18\linewidth}|p{0.16\linewidth}|p{0.16\linewidth}|p{0.16\linewidth}|p{0.16\linewidth}|}
\hline
\textbf{Dataset} &  \textbf{Access Portal} & \textbf{L0 res. (nm$^3$)}& \textbf{L1 res. (nm$^3$)} & \textbf{L0 Train Set Size (px.)}& \textbf{L0 Test Set Size (px.)} \\
\hline
CA1 somas &  BossDB \cite{hider_brain_2022} & $800\times 800 \times 800$ & $1.6e3\times 1.6e3 \times 1.6e3$ & $720 \times 938 \times 179$ & $718 \times 938 \times 179$ \\
\hline
fMOST somas &  BIL \cite{kenney_brain_2024} & $3.5e3 \times 3.5e3 \times 1e5$ & $3.5e3 \times 3.5e3 \times 1e3$ & $1410\times 1900\times 1000$ &  $1410\times 1900\times 1000$\\
\hline
C. elegans mito. &  BossDB \cite{hider_brain_2022} & $16 \times 16 \times 30$& $32 \times 32 \times 30$ & $4992 \times 2752 \times 80$ & $4992 \times 2752 \times 200$ \\
\hline
HeLa nucleus &  OpenOrganelle \cite{heinrich_whole-cell_2021} & $32 \times 32 \times 25.9$ & $64 \times 64 \times 51.8$ & $1500 \times 1550 \times 60$ & $1500 \times 1550 \times 64$  \\
\hline
H\&E tissue &  CAMELYON \cite{alexi_supporting_2018} & $1e3 \times 1e3$ & $2e3 \times 2e3$ & N/A & 10 images  \\
\hline
\end{tabular}
\end{table}

\subsection{Implementation}

For soma detection in fluorescence imaging, the detectors were ilastik pixel classification modules (random forests) \cite{berg_ilastik_2019}. In EM organelle segmentation, the detectors were U-nets with pretrained ResNet backbones \cite{ronneberger_u-net_2015}. For tissue segmentation, we used the ``HEST'' pretrained model via the TRIDENT package \cite{jaume_hest-1k_2025,zhang_accelerating_2025}. 

Connected components below a certain area threshold (in pixels) were filtered out after segmentation. Segmentations were converted to object detections by identifying connected components. Neural networks were executed on a GPU, and ilastik experiments were executed on a CPU. A summary of the implementation details is given in Table \ref{tab:methods}.

\begin{table}
\centering
\caption{Implementation details of experiments.}\label{tab:methods}
\begin{tabular}{|p{0.19\linewidth} | p{0.12\linewidth}|p{0.3\linewidth}|p{0.35\linewidth}|}
\hline
\textbf{Dataset} &  \textbf{Detector} & \textbf{L1:L0 Area Threshold (px.)} & \textbf{Hardware (memory)} \\
\hline
CA1 somas &  ilastik & 1000:50 & Apple M2 Pro (16GB) \\
\hline
fMOST somas &  ilastik & 20:5 & Apple M2 Pro (16GB) \\
\hline
C. elegans mito. &  U-Net & 36:10 & GeForce RTX 2080 Ti (12GB) \\
\hline
HeLa nucleus &  U-Net & 1000:250 & GeForce RTX 2080 Ti (12GB)  \\
\hline
H\&E tissue &  HEST & 0:0 & GeForce RTX 2080 Ti (12GB)  \\
\hline
\end{tabular}
\end{table}

\section{Results}

The results are organized in Table \ref{tab1}. On the CA1 somas dataset, both detectors detected six out of seven cell bodies, and had no false positives (Fig. \ref{fig:results}a). The cascade detector ran more than twice as fast as the single-level detector.

On the fMOST somas dataset, the cascade detector ran almost four times as fast, but had one more false positive and one more false negative than the single-level detector. Some of the detections are shown in Figure \ref{fig:results}b.

The detectors had identical accuracy on the C. elegans mitochondria dataset. Mitochondria was densely present in this dataset, so the primary advantage of the cascade detector was as a tissue detector that avoided calling the L0 classifier on blank margins (Fig. \ref{fig:results}c).

The cascade detector was over twice as fast on the HeLa nucleus dataset, with a reduction of only 0.01 in recall. Finally, the cascade detector was 40\% faster on the H\&E tissue dataset, and disagreed with the single-level detector on less than 1\% of pixels.

\begin{figure}
\includegraphics[width=\textwidth]{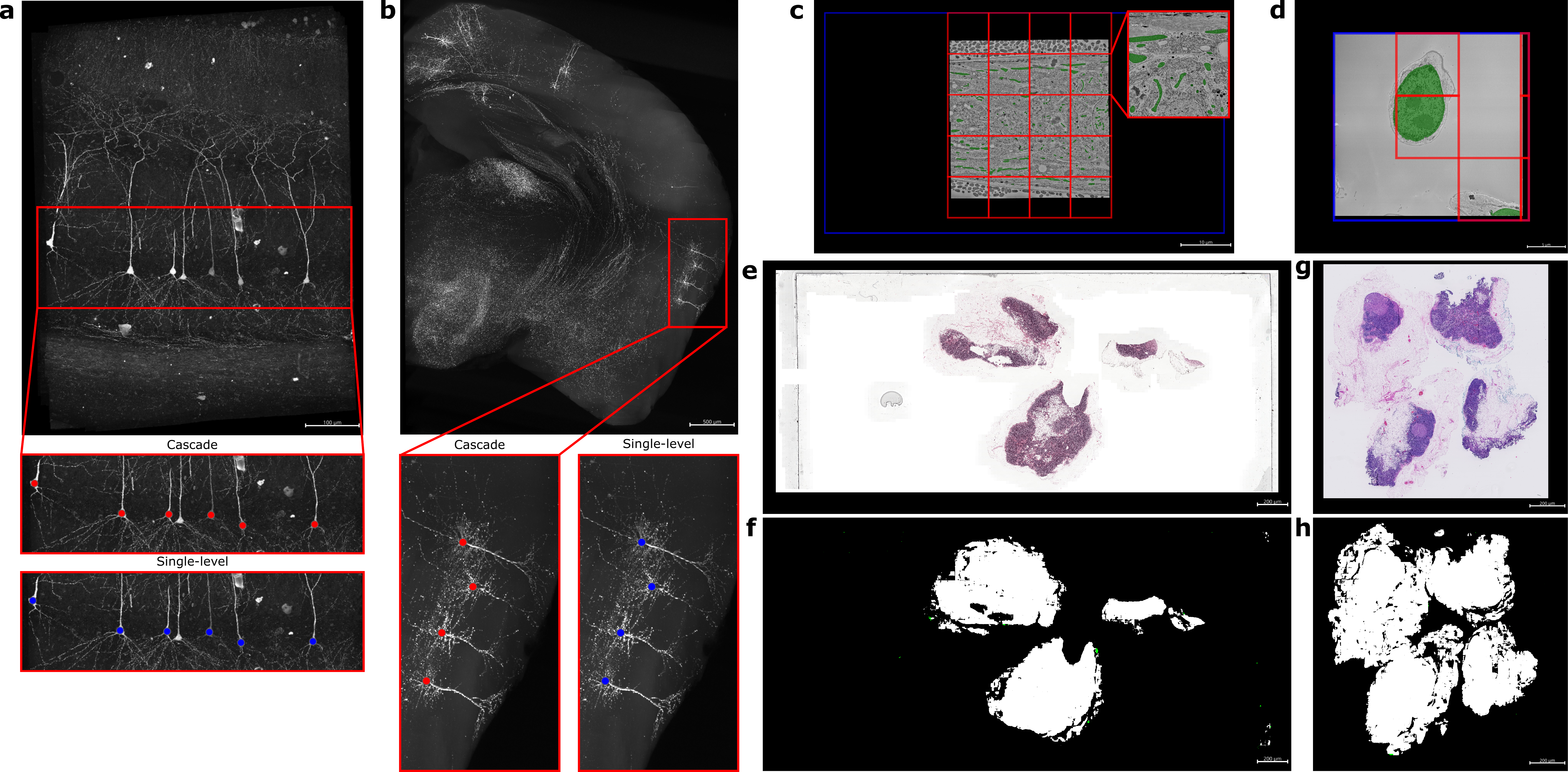}
\caption{A selection of results on the various datasets. \textbf{a} The test image from the CA1 somas dataset included seven fluorescent neurons \cite{bloss_structured_2016} with insets showing the detections from the cascade detector (red points) and single-level detector (blue points). \textbf{b} The test image from the fMOST somas dataset includes fluorescent neurons throughout mouse cortex \cite{zeng_whole_2024}. Insets show four neurons that were correctly detected by both the cascade detector (red points) and single-level detector (blue points). \textbf{c} A slice of the test image from the C. elegans mitochondria dataset are shown, with the ground truth segmentation overlaid in green, and the outer limits of the volume shown in blue \cite{witvliet_connectomes_2021}. Red grid lines depict the L1 subvolumes that were passed to the L0 detector. \textbf{d} A slice of the HeLa nucleus test set is shown, with the same overlays as in \textbf{c}. \textbf{e-h} Two test images from the H\&E tissue dataset are shown (\textbf{e,g}), with the detector segmentations (\textbf{f,h}). The tissue segmentation by the single-level detector (green) and cascade detector (magenta) are summed in color space where they overlap (white).} \label{fig:results}
\end{figure}



\begin{table}
\centering
\caption{Comparison of single-level and cascade detector on five datasets. In each case, results are reported from the held-out test set. For organelle semantic segmentation, recall and precision are computed pixelwise. For the H\&E dataset, accuracy is computed with respect to the single-level detector output. ``L1:L0 Calls'' denotes the number of calls to the level 1 or level 0 detector.}\label{tab1}
\begin{tabular}{|l|l|l|l|l|l|}
\hline
\textbf{Dataset} & \textbf{Detector} &  \textbf{Recall} & \textbf{Precision} & \textbf{L1:L0 Calls} & \textbf{Runtime (s)} \\
\hline
CA1 somas & Single-level &  0.86 & 1.0 & 0:96 & 837 \\
\hline
& Cascade  &  0.86 & 1.0 & 12:40 & 386 \\
\hline 
fMOST somas & Single-level &  0.61 & 0.85 & 0:1200 & 13484 \\
\hline
& Cascade  &  0.56 & 0.83 & 400:104 & 3396 \\
\hline 
C. Elegans mito. & Single-level &  0.81 & 0.55 & 0:44000 & 601 \\
\hline
& Cascade  &  0.81 & 0.55 & 12000:18024 & 382 \\
\hline
HeLa nucleus & Single-level &  0.91 & 0.95 & 0:2688 & 1263 \\
\hline
& Cascade  &  0.90 & 0.95 & 384:1088 & 569 \\
\hline
H\&E tissue & Single-level &  N/A & N/A & 0:38452 & 1749 \\
\hline
& Cascade  & >0.99 & 1.0  & 9681:11792 & 1043 \\
\hline
\end{tabular}
\end{table}

\section{Discussion}

In this work we show that cascade detectors, which have been popular in face detection applications, can be successfully applied to microscopy data for tasks such as fluorescent soma detection, organelle segmentation, and tissue segmentation. Under a conditional independence assumption, we show that the number of inference calls can be dramatically reduced in the cascade setting if the object of interest is sparse. These results extend to signals on a domain of arbitrary dimension, and our experiments include both two-dimensional and three-dimensional images, and both random forest and neural network detectors.

The main limitation in our theoretical work is our assumption that the performance of detectors at different levels are independent, conditioned whether an object exists in that region. In practice, detectors at different levels might be correlated. For example, inherent difficulty of different samples might lead to correlation between the detectors. A model proposed by Eckhardt and Lee \cite{eckhardt_theoretical_1985}, and summarized by Hansen and Salamon \cite{hansen_neural_1990} incorporates sample-dependent error rates. Our accuracy result in Proposition \ref{prop:acc} would change by allowing all variables to depend on the sample $x$ (e.g. $\beta_1$ becomes $\beta_1(x)$), then integrating over the probability distribution of $x$. We note that the number of classifier calls from Proposition \ref{prop:acc} does not change, since it involves only the L1 detector.

Another potential concern is that constructing a cascade detector involves training detectors at different resolutions. It is theoretically possible to use the same L0 training data to train higher level detectors by downsampling the annotations, or to use a single detector that supports multiple resolutions \cite{zhang_real-time_2007,ren_faster_2015}. However, these approaches would likely lead to statistical dependence mentioned earlier. In practice, the decision on if and how to train detectors at various resolutions depends on the number of training samples needed, the cost of obtaining annotations, and the extent to which statistical dependence affects accuracy.

There are several potential future directions stemming from this work. The first is comparing predicted performance from Proposition \ref{prop:acc} to actual performance. The detectors could be applied to a validation dataset to estimate the sensitivities and specificities used in the given formulas. Further, one could estimate statistical dependence between detectors in practice by computing correlations between their predictions.

The cascade method showed a larger speedup on three-dimensional images. This agrees with Remark \ref{remark:dim} where, in the limit of increasing sparsity, the expected number of classifier calls decreases with domain dimension. We are also interested to see if this trend extends to other dimensions such as time-series data (one-dimensional), or data that varies in both space and time. Setting up a cascade detector with more than two levels is also a potential area of future work.

This work is orthogonal to other areas of research into efficient inference such as sparsification and quantization, which focus on the arithmetical implementation of neural networks. Those methods could be combined with this work to further accelerate image analysis. We believe this work has broad utility in making image analysis in microscopy data more efficient since it is agnostic to the type of detector used, the domain dimensionality, and the object of interest.

%
%

\newpage
\bibliographystyle{splncs04}
\bibliography{references}

\end{document}